\documentclass{article}

\usepackage{log_2024}

\usepackage[utf8]{inputenc} 
\usepackage[T1]{fontenc}    
\usepackage{url}            
\usepackage{booktabs}       
\usepackage{amsfonts}       
\usepackage{nicefrac}       
\usepackage{microtype}      
\usepackage{xcolor}         
\usepackage{subfigure}
\usepackage{mathtools}

\usepackage{amsmath}
\usepackage{amssymb}
\usepackage{mathtools}
\usepackage{amsthm}
\usepackage[numbers, square]{natbib}
\usepackage{xspace}
\usepackage{algorithm}
\usepackage{algpseudocode}

\algnewcommand\algorithmicforeach{\textbf{for each}}
\algdef{S}[FOR]{ForEach}[1]{\algorithmicforeach\ #1\ \algorithmicdo}

\newcommand{\sgn}{\operatorname{sgn}}

\newcommand{\mlp}{\operatorname{m}}

\newcommand{\1}{\mathbf{1}}
\newcommand{\I}{\mathbf{I}}

\newcommand{\R}{\operatorname{\mathbb{R}}}
\newcommand{\N}{\operatorname{\mathbb{N}}}

\newcommand{\bA}{\operatorname{\mathbf{A}}}

\newcommand{\bD}{\operatorname{\mathbf{D}}}
\newcommand{\bH}{\operatorname{\mathbf{H}}}

\newcommand{\bP}{\operatorname{\mathbf{P}}}

\newcommand{\bS}{\operatorname{\mathbf{S}}}

\newcommand{\bX}{\operatorname{\mathbf{X}}}

\newcommand{\ba}{\operatorname{\mathbf{a}}}

\newcommand{\bd}{\operatorname{\mathbf{d}}}

\newcommand{\bh}{\operatorname{\mathbf{h}}}
\newcommand{\bp}{\operatorname{\mathbf{p}}}

\newcommand{\bs}{\operatorname{\mathbf{s}}}

\newcommand{\bx}{\operatorname{\mathbf{x}}}
\newcommand{\by}{\operatorname{\mathbf{y}}}

\newcommand{\low}{\text{A}}
\newcommand{\band}{\text{C}}

\newtheorem{thm}{Theorem}
\newtheorem*{thm*}{Theorem}

\newtheorem*{lem*}{Lemma}

\theoremstyle{definition}
\newtheorem{defn}{Definition}
\newtheorem*{defn*}{Definition}

\newtheorem*{ex*}{Example}

\newcommand{\method}{GCON\xspace}
\title{Towards a General Recipe for Combinatorial Optimization with Multi-Filter GNNs}

\author[F. Wenkel et al.]{%
  Frederik Wenkel\thanks{Equal contribution.}
    \\
    Dept. of Mathematics \& Statistics\\
    Université de Montréal\\
    Mila – Quebec AI Institute\\
  \texttt{frederik.wenkel@umontreal.ca} \\
  \And
  Semih Cantürk\footnotemark[1]
    \\
    Dept. of Computer Science and Operations Research\\
    Université de Montréal\\
    Mila – Quebec AI Institute\\
  \texttt{semih.canturk@umontreal.ca} \\
    \And
  Stefan Horoi
    \\
    Dept. of Mathematics \& Statistics\\
    Université de Montréal\\
    Mila – Quebec AI Institute\\
  \And
  Michael Perlmutter\thanks{Equal senior author contributions.}
    \\
  Dept. of Mathematics\\
  Boise State University\\
  \And
  Guy Wolf\footnotemark[2]
    \\
    Dept. of Mathematics \& Statistics\\
    Université de Montréal\\
    Mila – Quebec AI Institute\\
}

\begin{document}

\maketitle
\setcounter{footnote}{0}

\begin{abstract}
Graph neural networks (GNNs) have achieved great success for a variety of tasks such as node classification, graph classification, and link prediction. However, the use of GNNs (and machine learning more generally) to solve combinatorial optimization (CO) problems is much less explored. Here, we introduce \method, a novel GNN architecture that leverages a complex filter bank and localized attention mechanisms to solve CO problems on graphs. We show how our method differentiates itself from prior GNN-based CO solvers and how it can be effectively applied to the maximum cut, minimum dominating set, and maximum clique problems in a unsupervised learning setting. \method is competitive across all tasks and consistently outperforms other specialized GNN-based approaches, and is on par with the powerful Gurobi solver on the max-cut problem. 
We provide an open-source implementation of our work\footnote{\url{https://github.com/WenkelF/copt}}.
\end{abstract}

\section{Introduction}
\label{sec:intro}
Recent years have seen the rapid development of neural networks for graph-structured data \citep{wu2022graph,abadal2021computing}. Indeed, graph neural networks (GNNs) achieve state-of-the-art performance for tasks such as node classification and are incorporated in industrial applications such as helping power Google Maps~\citep{derrow2021eta} and Amazon's recommender system \citep{Wang2022}. Popular message-passing neural networks (MPNNs) typically consider a graph together with a collection of features associated with each node~\citep{kipf2016semi,xu2018powerful,velickovic2017graph}. In each layer, they first perform a local aggregation which effectively smooths the features over neighboring nodes of the graph, and then perform a transformation which learns combinations of the smoothed node features. 

Here, we consider leveraging GNNs to solve combinatorial optimization (CO) problems, a somewhat less explored yet emerging application area of graph learning. In particular, we focus on finding (i) the maximum clique, i.e., the largest fully connected subgraph, (ii) the minimum dominating set, i.e., the smallest subset of the vertices which ``dominates" the graph in the sense that every vertex is within at most one hop of the dominating set, and (iii) the maximum cut, i.e., a partition of the vertices into two clusters with as many cross-cluster edges as possible.

Notably, these tasks differ significantly from, e.g., node classification in several important ways. The aggregations used in common message passing neural networks are localized averaging operations that can be interpreted, from the perspective of graph signal processing (GSP) \citep{ortega2018graph,shuman2013emerging}, as low-pass filtering. In this manner, they effectively treat smoothness as an inductive bias. While this is a useful heuristic on homophilic social networks such as the ubiquitous Cora, CiteSeer, and PubMed datasets~\citep{bojchevski2018deep}, we find that this assumption is not quite appropriate for CO problems. Moreover, in CO problems, it is not reasonable to assume the presence of informative node features. Instead, algorithms are expected to learn directly from the network geometry (taking initial inputs to be constant or features derived from the graph structure, e.g., vertex degree). We also note that the CO problems we are interested in are typically NP-hard. Therefore, it is computationally intractable to obtain sufficient labeled data for supervised learning. As we shall see below, our approach tackles various CO problems with a fully unsupervised framework, using task-specific loss functions.

To address these challenges, we propose the \textcolor{black}{Graph Combinatorial Optimization Network} (\method), a novel GNN that uses a hybrid filter bank consisting of both (i) aggregation operations such as those used in common MPNNs, as well as (ii) comparison operations, which will take the form of band-pass, wavelet filters and aim to capture more intricate aspects of the network geometry. In order to balance the importance of these different filters, we use a localized attention mechanism which may choose to lend more importance to different filters on a node-by-node basis. 
We compare \method to various competing methods, such as heuristics, commercial solvers, other specialized GNN-based approaches~\citep{karalias2020erdos, min2022can}, and recent advances that apply GFlowNets (GFNs) to the above mentioned problems~\citep{zhang2023let}. 
Our main contributions are: 

\begin{enumerate}
    \item We develop \method, a novel GNN architecture with a sophisticated filter bank and a localized attention mechanism that is particularly suited for solving CO problems.
    \item We evaluate \method on a variety of graph-CO problems on synthetic graph benchmarks, and demonstrate that our architecture is highly performant.
    \item We prove a theorem that helps us interpret the improved performance of our method.
\end{enumerate}

\section{Background and Related Work}
\label{sec:related}

Many well-known graph-CO problems are either NP-hard or even NP-complete \citep{karp72reducibility}, rendering it impossible to efficiently generate exact solutions for large graphs. Instead, one may hope to efficiently generate \emph{approximate solutions}. Almost forty years after \citet{Hopfield1985} first proposed applying deep neural networks to CO problems, there has been a recent revival in achieving this goal via GNNs. Indeed, GNNs are a natural candidate due to their success on a wide variety of tasks such as node classification, graph classification, and link prediction \citep{abadal2021computing,wu2022graph}; a comprehensive overview can be found in \citet{co_gnns}. Researchers have explored ways to solve graph-CO problems using both supervised and unsupervised (in particular self-supervised) methods as well as reinforcement learning (RL). Unfortunately, supervised methods are often computationally infeasible due to the cost of labeling for NP-hard problems. RL-based solutions, on the other hand, may run into problems arising from the large size of the state space. We thus view self-supervised learning as the most promising family of approaches. 

Notably, \citet{tonshoff2019run, karalias2020erdos, min2022can} have all used GNNs to solve CO problems in a unsupervised manner through self-supervision. \citet{tonshoff2019run} solves binary constraint satisfaction problems using message-passing with recurrent states. Subsequently, \citet{karalias2020erdos} proposed a seminal self-supervised learning framework in which a GNN outputs a Bernoulli distribution over the vertex set of the input graph, and a surrogate loss function designed to encourage certain properties in the distribution (e.g., vertices that are assigned high probabilities form a clique) and softly enforce constraints in a differentiable manner. They then propose specific loss instantiations to tackle maximum clique and graph partitioning problems. A similar loss is derived in \citet{min2022can}, who use an architecture based on the graph scattering transform \citep{gao2019geometric} to obtain promising results for the maximum clique problem. A notable difficulty associated with these methods is that the proposed self-supervised loss functions are highly non-convex. This motivated a separate line of work by \citet{sun2022annealed} which designed an annealed optimization framework that aims to optimize GNNs effectively for such loss functions. We also note the recent work by \citet{zhang2023let}, which effectively applied generative flow networks~\citep{gflownets} to solve CO problems in an iterative manner.
Our work is most closely related to \citet{karalias2020erdos} and \citet{min2022can}, relying on the same or closely related self-supervised loss functions, while improving on the learned part of the pipeline, leveraging a more powerful hybrid GNN architecture.

Finally, \method is related to several lines of work on building neural solvers that can perform well across graph-CO problems. \citet{anycsp} and \citet{boisvertcsp} propose building unified representations for constraint satisfaction problems (CSP), \citet{berto2024routefinderfoundationmodelsvehicle} and \citet{multitask_vrp} similarly propose unified methods for vehicle routing problem (VRP) variants, while \citet{drakulic2024goalgeneralistcombinatorialoptimization} proposes a modular architecture for multi-task learning over a variety of CO problems.

\section{Problem Setup}
\label{sec:problem_setup}
The purpose of this paper is to develop methods for (approximately) solving combinatorial optimization problems on graphs in a scalable and computationally efficient manner. 

We let $G=(V, E)$ denote a simple, connected, undirected graph. We assume an arbitrary (but fixed) ordering on the vertices $V=\{v_1,\ldots,v_n\}$. In a slight abuse of notation, if $\bx$ is a function defined on the vertices $V$, we do not distinguish between $\bx$ and the vector defined by $x_i=\bx[v_i]$. We let $\bA\in\R^{n\times n}$ be the adjacency matrix of $G$, let $\bD$ be the corresponding diagonal degree matrix, $\bD=\text{diag}(\bA\1)$, and let $\bP=\frac{1}{2}(\I+\bA\bD^{-1})$ denote the transition matrix of a lazy random walk (a Markov chain in which, at each step, the walker either stays put or moves to a neighboring vertex, each with probability one half). Finally, we let $\bX\in\R^{n\times d}$ denote a matrix of input node features.

We consider the following CO problems in this work, but also note that our framework could be extended to a large variety of graph-CO problems such as vertex cover, maximal independent set or graph coloring by constructing appropriate supervised or self-supervised loss functions:
\begin{enumerate}
    \item \textbf{Maximum Cut:} Find a partition of the vertex set $V=S\sqcup T$ that maximizes the number of edges between $S$ and $T$, i.e., $|\{\{v_i,v_j\}\in E: v_i\in S, v_j\in T\}|$.
    \item \textbf{Maximum Clique:} Find a fully connected subgraph (a clique) with as many vertices as possible.
    \item \textbf{Minimum Dominating Set:} A subset of the vertices $S\subseteq V$ is called a \emph{dominating set} if for every vertex $v_i$ is either an element of $S$ or a neighbor of $S$, in the sense that $\{v_i,v_j\}\in E$ for some $v_j\in S$. Find a dominating set with as few vertices as possible.
\end{enumerate}

\section{Methodology}
\label{sec:method}

We treat each of the CO problems as an unsupervised learning problem. We first compute basic node statistics to serve as input features to the graph. We then use our GNN, described in detail in Section~\ref{sec:model}, to generate a function/vector $\bp$ whose $i$-th entry $\bp_i=\bp[v_i]$ is interpreted as the probability that $v_i$ is in the set of interest. We apply a sigmoid in its final layer ensuring that all entries are between $0$ and $1$. Finally, we use a \emph{rule-based decoder} to construct a candidate solution by subsequently adding nodes ordered by decreasing probability, based on $\bp$. Our methodology is illustrated in Fig.~\ref{fig:abs_fig}. In the following subsections, we will explain how to develop suitable loss functions and decoders for each problem of interest.

\subsection{Maximum Cut}
\label{sec:mcut}
In this problem, we aim to find the maximum cut $V=S\sqcup T,$ and without loss of generality, we choose $p_i$ to be the probability that $v_i\in S$. We then define $\by \coloneqq 2\bp - \1_n$ (where $\1_n\in\mathbb{R}^n$ is the vector of all 1's) so that the entries of $\by$  lie in $[-1,1]$ (rather than $[0,1]$) and use the loss function
\begin{equation*}
\textstyle
    L(\by) \coloneqq \frac{1}{2}\by^\top \bA \by = \sum_{\{v_i,v_j\}\in E} y_i y_j.
\end{equation*}
After learning $\by$, we then set $S$ to be the set of vertices where $\by$ is greater than or equal to zero and $T$ to be the set of vertices where $\by$ is negative.
Intuitively, we note that every edge that connects two nodes in different parts of the cut, i.e., $\sgn(\by_i) = - \sgn(\by_j)$, decreases the loss, while the opposite case yields an additive term that increases the loss. For example, if $\bp=\1_S$ (where $\1_S$ is the indicator function on $S\subset V$, i.e, $\1_S[v]=1$ if $v\in S$, and zero otherwise), we have $\by=\1_S-\1_T$, and our loss function is equal to the number of cuts minus the number of intra-cluster edges. Thus, the loss function encourages a node labeling that cuts as many edges as possible. 

\subsection{Maximum Clique}
\label{sec:mc}
For the maximum clique, we let $\overline{\bA}$ denote the adjacency matrix of the complement graph $\overline{G}=(V,E^c)$, where $E^c$ is the set of all $\{\{v_i,v_j\}\notin E,$ $i\neq j$\} and use the two-part loss function considered in ~\citet{min2022can} (see also~\citet{karalias2020erdos}): 
\begin{equation*}
      L(\bp) \coloneqq L_1(\bp) + \beta L_2(\bp) = - \bp^\top {\bA} \bp + \beta \bp^\top \overline{\bA} \bp.
\end{equation*}

To understand this loss function, consider the idealized case where $\bp=\1_S$ is the indicator function of some set $S$. In that case, one may  verify that 
$$
-L_1(\1_S)=\1_S^\top\bA\1_S=\sum_{i,j}\bA[i,j]\1_S[v_i]\1_S[v_j]
$$
is the number of edges between elements of $S$. 
Similarly, we see have that 
$$
L_2(\1_S)=\1_S^\top\overline{\bA}\1_S=\sum_{i,j}\overline{\bA}[i,j]\1_S[v_i]\1_S[v_j]
$$
is the number of ``missing edges'' within $S$, i.e., the number of $\{v_i,v_j\}\in E^c,$ $v_i,v_j\in S$.
Therefore, in minimizing $L(\bp)$, we aim to find a subset with as many connections, and as few missing connections, as possible. 
The hyper-parameter $\beta$ is used to balance the contributions of $L_1$ and $L_2$.

After computing $\bp$, we then reorder the vertices $\{v_i\}_{i=1}^n\rightarrow\{v_i'\}_{i=1}^n$ so in the new ordering $\bp(v_i')\geq \bp(v_{i+1}')$. We then build an initial clique by initializing $C^{(1)}=\{v_1'\}$ and iterating over $i$. At each step, if $C^{(1)}\cup \{v_i'\}$ forms a clique, we add $v_i'$ to $C^{(1)}$, otherwise we do nothing. Note that $C^{(1)}$ is guaranteed to be a clique by construction.

We then repeat this process multiple times to construct additional cliques $C^{(2)},\ldots, C^{(K)}$.  However, on the $k$-th iteration, we initialize with $C^{(k)}=\{v_k'\}$ and automatically exclude $v_1',\ldots, v_{k-1}'$ from $C^{(k)}$. We then choose the maximal clique $C^*$ to be the $C^{(k)}$ with the largest cardinality. We note that increasing $K$ does increase the computational cost of our decoder; however, running the decoder with different initializations $\{v_k'\}$ can be parallelized in a straightforward manner. 

\subsection{Minimum Dominating Set}
\label{sec:mds}
For the minimum dominating set problem, we use the loss function: 
\begin{equation*}
L(\bp)\coloneqq L_1(\bp)+\beta L_2(\bp)\coloneqq \Vert\bp\Vert_1 + \beta \sum_{v_i\in V} \exp \left(\sum_{v_j \in \mathcal{N}_{\underline{v_i}}} \log(1-\bp[v_j])\right),
\end{equation*}
where $\mathcal{N}_{\underline{v}}$ denotes the set of nodes within distance one of $v$ (including $v$ itself).
The first loss term promotes the sparsity of the dominating set.
To understand the second term, observe as soon as there is at least one $v_j\in \mathcal{N}_{\underline{v_i}}$ with $\bp[v_j]\approx 1$ we will have $\exp\left(\sum_{v_j \in \underline{\mathcal{N}_{v_i}}} \log(1-\bp[v_j])\right) \approx 0$.  Indeed, we note that, if $\bp=\1_S$ is the indicator function of a dominating set $S$, we have that 
$\exp\left(\sum_{v_j \in \mathcal{N}_{\underline{v_i}}} \log(1-\bp[v_j])\right)=\exp(-\infty)=0,$
for all $v_i$. Therefore, $L(\1_S)=\|\1_S\|_1$ would be exactly the cardinality of $S$.

After learning $\bp$, we construct our proposed minimum dominating set in a manner analogous to the max-clique case. We again begin by reordering the vertices so that $\bp(v_i')\geq \bp(v_{i+1}')$. We then initialize our first dominating set $S^{(1)}=\{v_1'\}$ and iterate over $i$. At each step, we check if $S$ is a dominating set. If so, we stop. Otherwise, we add $v_{i+1}'$ to $S^{(1)}$ and repeat the process. Similarly to the maximum clique, we repeat this process several times, where we initialize $S^{(k)}=\{v_k'\}$ (automatically excluding $v_1',\ldots,v_{k-1'}$ from the set $S^{(k)}$ considered on the $k$-th iteration) and take $S^*$ to be the smallest dominating set $S^{(k)}$ found in any iteration.

\begin{figure}[t]
\vspace{-10pt}
    \centering
    \includegraphics[width=\textwidth]{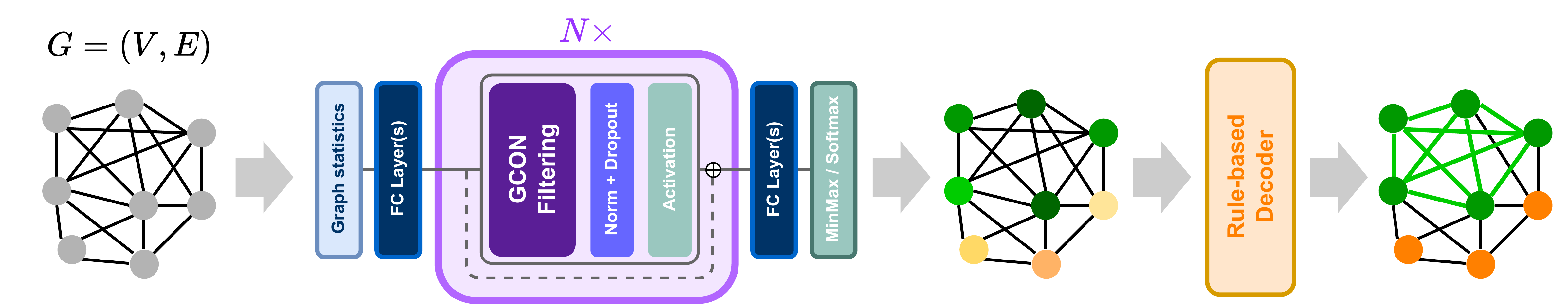}
    \caption{Illustration of our general framework. Graphs are first equipped with node features derived from graph statistics and then passed through a GNN composed of \method hybrid layer blocks (detail in Fig.~\ref{fig:new}) with a min-max or softmax output layer. A rule-based decoder then processes the node-level outputs to determine the set of interest.}
    \label{fig:abs_fig}
\end{figure}

\vspace{-10pt}
\section{The \method Architecture}
\label{sec:model}
In this section, we introduce the \textcolor{black}{Graph Combinatorial Optimization Network} (\method). A notable feature of our architecture is that it utilizes several different types of filtering operations in each layer.
It includes localized aggregation operations similar to those used in common message-passing neural networks (which may be interpreted as low-pass filters from a GSP perspective). We also include comparison operations that extract multi-scale geometric information and consider changes across the different scales. These comparison operations will take the form of wavelets, based on those utilized in the geometric scattering transform \citep{gao2019geometric,gama2018diffusion,zou2020graph}, and can be interpreted as band-pass filters again from a GSP perspective. The different filters of each type will be combined using a localized attention mechanism that allows the network to focus on different filters at each node, which makes our architecture fundamentally different from many traditional GNN methods that apply the same aggregation at each node of the graph. 

\begin{figure}[t]
    \vspace{-10pt}
    \subfigure[]{\label{fig:new}
        \hspace{-4pt}\includegraphics[height=0.22\textheight]{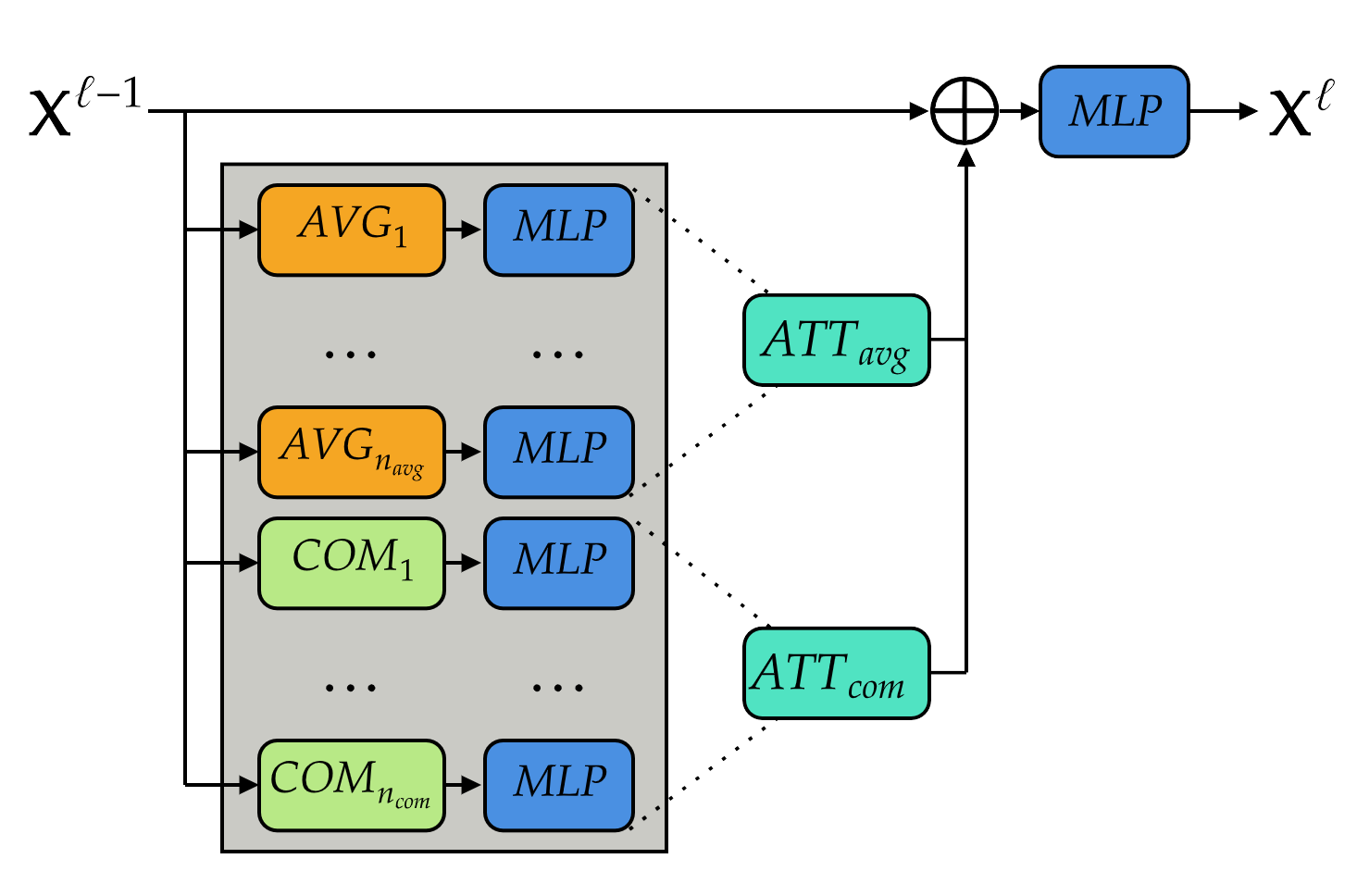}
    }
    \subfigure[]{\label{fig:old}
        \hspace{2pt}\includegraphics[height=0.22\textheight]{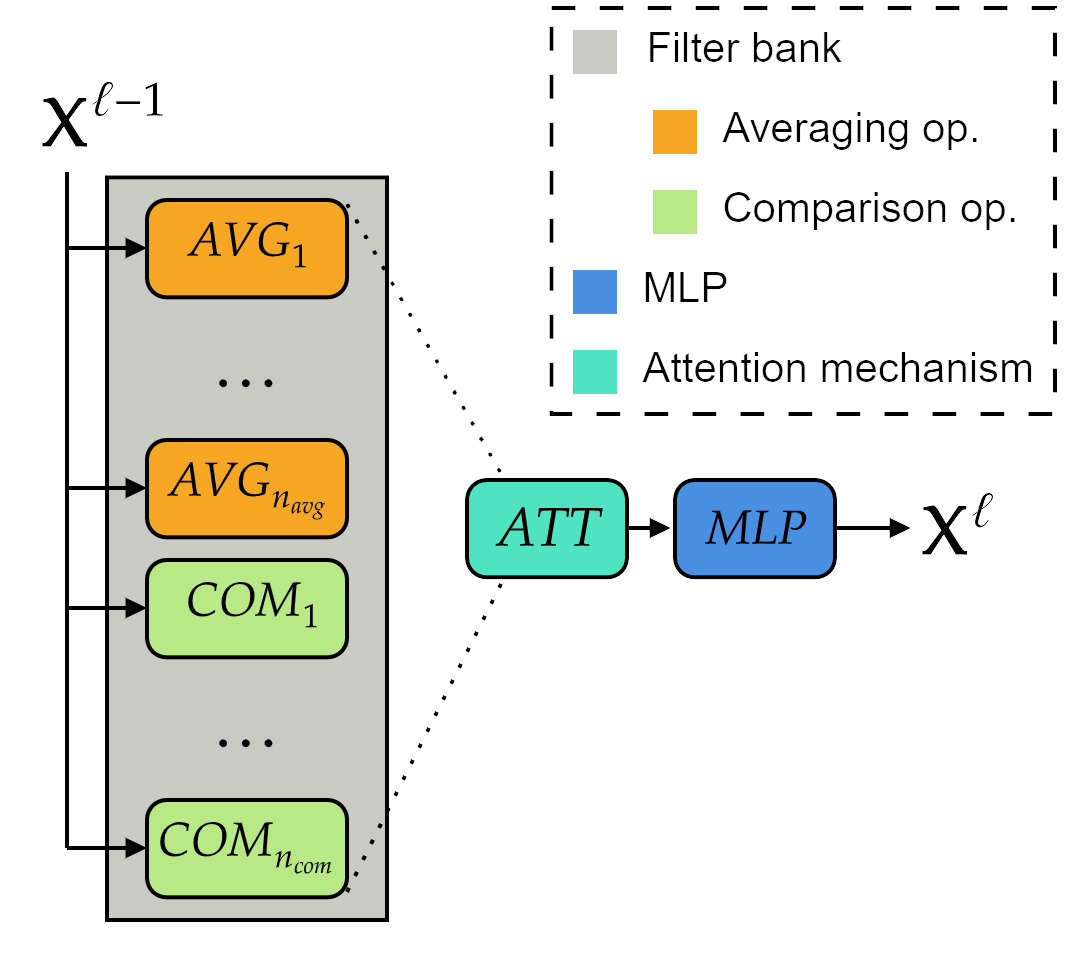}
    }
    \caption{Illustration of the layer-wise update for (a) the new decoupled \method filter bank and (b) the ScatteringClique filter bank from \citet{min2022can}.}
\end{figure}
\subsection{The \method filter bank: Aggregation and Comparison operations}
The most basic component of our filter bank is a traditional one-step aggregation somewhat similar to popular message passing neural networks (e.g., ~\citet{kipf2016semi, hamilton2017inductive}). It has the form
    $F_1(\bX) \coloneqq \mlp\left( \bP \bX \right)$,
where $\bP=\frac{1}{2}(\I+\bA\bD^{-1})$ is the lazy random walk matrix.\footnote{One could readily replace $\bP$ with generic graph shift operators $\bS\in\R^{n\times n}$, i.e., a matrix where $\bS[i,j]=0$ unless $\{v_i,v_j\}\in E$ or $i=j$. Likewise with the more complex filters $F_k(\bX)$ and $F_{k_1,k_2}(\bX)$ defined below.} The function $\mlp \coloneqq \mlp_{\sigma, \theta}$ can be either a linear layer or multi-layer perceptron (MLP) with activation function $\sigma$ and learned weights $\theta$.
At each node, this process averages information from all direct neighbors of the node, followed by a transformation step that learns new cross-feature combinations. Networks utilizing this setup are commonly referred to as \emph{aggregate-transform} GNNs \citep{xu2018powerful} due to the alternation of \emph{aggregations} $\bX\rightarrow \text{AGG}(\bX)$ and \emph{transformations} of node features $\text{AGG}(\bX)\rightarrow\mlp(\text{AGG}(\bX))$. Additionally, we also consider higher-order aggregations of the form
\begin{equation}\label{eq:low-k}
    F_k(\bX) \coloneqq \mlp\left(\bP^k \bX \right),\quad k\geq 1.
\end{equation}
We will refer to the filters $F_k$ as \emph{aggregation operations}.

Next, we add filters inspired by diffusion wavelets~\citep{coifman2006diffusion} that calculate the difference of two different aggregations $F_{k_1}, F_{k_2}$, where $k_1\neq k_2$, by setting
\begin{equation}\label{eq:band}
    F_{k_1,k_2}(\bX) 
    \coloneqq \mlp\left( (\bP^{k_1} - \bP^{k_2}) \bX \right).
\end{equation}
Filters of this form are fundamentally different from $F_k(\bX)$ discussed above. From the GSP perspective, $F_{k_1,k_2}(\bX)$ constitute band-pass filters, whereas $F_k(\bX)$ constitute low-pass filters. 
Since the operation $\bX\rightarrow (\bP^{k_1}-\bP^{k_2})\bX$ compares aggregations at two different scales, we will refer to them as  \emph{comparison operations}. We denote the set of all our filtering operations as $\mathcal{F}=\mathcal{F}_\low\sqcup\mathcal{F}_\band$, where $\mathcal{F}_\low$ consists of aggregation operations and $\mathcal{F}_\band$ consists of comparison operations. As our filter bank consists of two different types of filters, we will also refer to it as a hybrid filter bank.

The use of the $F_{k_1,k_2}$ is inspired by the geometric scattering transform~\citep{zou2020graph, gao2019geometric, perlmutter2023understanding}, a handcrafted multi-layer network, which iteratively filters the input node features via dyadic diffusion wavelets of the form $\bX\rightarrow (\bP^{2^{j-1}}-\bP^{2^j})\bX$
and takes vertex-wise absolute values in between wavelet filterings. Notably, the original versions of the geometric scattering transform were handcrafted feature extractors. However, subsequently, \citet{wenkel2022overcoming,tong2022learnable} used the geometric scattering transform as a basis for a fully learned graph neural network. In particular, \citet{tong2022learnable} aimed to learn the optimal diffusion scales (i.e., the powers $k_i$ to which $\bP$ is raised)  and \citet{wenkel2022overcoming} introduced a \emph{hybrid-scattering network} that utilized both aggregation operations (low-pass filters) and comparison operations (band-pass, wavelet filters).   

\subsection{Attention over filters}\label{sec: Attention over filters}
We let $\bX^{\ell-1}$ denote the input to the $\ell$-th layer and use an attention mechanism to determine the importance of information from different filters for each node. For each filter response $\bH_f\coloneqq F(\bX^{\ell-1})\in \R^{n\times d}$, we concatenate it with a matrix of  transformed input features $\bH\coloneqq \mlp(\bX^{\ell-1})$,
and apply an attention mechanism to their horizontal concatenation $[\bH \Vert \bH_f]$. This takes the form
\begin{equation*}
    \bs_f^A \coloneqq \sigma\left( [\bH \Vert \bH_f] \ba_A \right), f\in \mathcal{F}_A, \quad \text{and} \quad \bs_f^C \coloneqq \sigma\left( [\bH \Vert \bH_f] \ba_C \right), f\in \mathcal{F}_C
\end{equation*}
where $\ba_A, \ba_C\in \R^{2d}$ are learned attention vectors for aggregation and comparison filters, respectively. 
We then use softmax to normalize the importance scores across filters within $\mathcal{F}_\low$ and $\mathcal{F}_\band$ at each vertex, resulting in a normalized score $\bar\bs^\low_f\in \R^n$ and $\bar\bs^\band_f\in \R^n$ given by
\begin{equation}\label{eq:att-2-decoupled}
    \bar\bs^\low_f(v) = \text{softmax}\left(\{\bs_f[v]: f\in \mathcal{F}_\low\}\right),\quad\bar\bs^\band_f(v) = \text{softmax}\left(\{\bs_f[v]: f\in \mathcal{F}_\band\}\right).
\end{equation}
Importantly, $\bs_f^{\low}$ and $\bs_f^{\band}$ (and thus  each $\bar\bs_f^{\low}$ and $\bar\bs^\band_f$) will take different values for each node. Therefore, they act as localized attention mechanisms, which can up-weight different filters at different vertices.

Next, the filter responses are reweighted according to the attention scores, yielding aggregation and comparison representations $\bH_{\mathcal{F}_\low}, \bH_{\mathcal{F}_\band}\in\R^{n\times d}$:
\begin{equation*}
    \textstyle
    \bH_{\mathcal{F}_\low} = \sum_{f\in \mathcal{F}_\low} \left(\bar\bs_f^{\low} \circ \1_d\right) \odot \bH_f, \quad\text{and}\quad\textstyle
    \bH_{\mathcal{F}_\band} = \sum_{f\in \mathcal{F}_\band} \left(\bar\bs_f^{\band} \circ \1_d\right) \odot \bH_f,
\end{equation*}
where, $\circ$ and $\odot$ denote the outer and Hadamard product, respectively. Finally, $\bX^\ell$ is then given by  
\begin{equation}\label{eq:att-4-}
    \bX^\ell = \text{MLP}\left(\bX^{\ell-1} + \bH_{\mathcal{F}_\low} + \bH_{\mathcal{F}_\band} \right).
\end{equation}

 Our layer-wise update rule is inspired by the one utilized in \citet{min2022can}, which also used a localized attention mechanism to balance the contributions of aggregation and comparison operations (i.e., low-pass and band-pass filters).  However, our network improves upon
 \citet{min2022can} in several ways. Most importantly, it computes separate normalized attention scores $\bar\bs^\low_f(v)$ and $\bar\bs^\band_f(v)$ for the aggregation and the comparison operators, whereas \citet{min2022can} uses a single softmax for the entire filter bank $\mathcal{F}$. Additionally, \citet{min2022can} omits the learnable function $\mlp$ in Equations~\ref{eq:low-k} and~\ref{eq:band}, and uses a different layer-wise update in place of Equation~\ref{eq:att-4-}. 
For further details, see Appendix~\ref{sec: original}.

\subsection{Theoretical Analysis}\label{sec:motivation}

As noted in the previous subsection, one of the improvements we introduce relative to \citet{min2022can} is a different method for computing the normalized attention scores.
We normalize separately over $\mathcal{F}_\low$ and $\mathcal{F}_\band$, resulting in $\bar\bs^{\low}$ and $ \bar\bs^{\band}$, where \citet{min2022can} applies a single attention mechanism over the entire filter bank $\mathcal{F}=\mathcal{F}_\low\sqcup\mathcal{F}_\band$. Below, we provide a theoretical analysis that underlines the advantages of this modification.  For clarity, we refer to our method as the \emph{decoupled architecture} since it separates (i.e., decouples) the aggregation operations $\mathcal{F}_A$ and the comparison operations $\mathcal{F}_C$ when computing the normalized attention scores. We also refer to a variant of our method that computes attention over the entire filter bank together (similar to \citet{min2022can}) as the \emph{non-decoupled architecture}. 
In our analysis of the non-decoupled architecture, we also assume that the function $\mlp(.)$ is not learned (only applying a nonlinear activation as in \citet{min2022can}). 


To understand the decoupled architecture, we note that the low-pass filters $F_k$ and band-pass filters $F_{k_1,k_2}$ constitute fundamentally different operations and capture different types of information. While low-pass filters use a weighted averaging of the features of the node's neighbors, band-pass filters \emph{compare} two such averages of neighborhoods at different scales. As different types of information may be useful at different nodes, it is important that our MLP has access to both families of filterings.

However, we find that when grouped together, the outputs of the low-pass filters tend to have a larger magnitude than the band-pass filters. Therefore, the MLP is incapable of extracting band-pass information in the non-decoupled architecture. In the remainder of this section, we prove a theorem illustrating this phenomenon as a complement to our empirical evidence presented later. 
We will prove a result demonstrating that, in the non-decoupled architecture, the low-pass information will dominate the band-pass information as long as the receptive fields of the filters are sufficiently large. For simplicity, we will assume that we have a single node feature so that the input to our layer is a signal $\mathbf{x}\in\mathbb{R}^n.$ Moreover, for $f\in\mathcal{F},$ we will write $\mathbf{h}_f[v]$ to denote the response of $f$ at v.

\vspace{4pt}
\begin{defn}[Band-dominant representation]\label{def:diverse}
    Consider a graph $G=(V, E)$, a non-negative input signal $\bx\in\R^n$, and a filter bank $\mathcal{F} = \mathcal{F}_\low \sqcup \mathcal{F}_\band$. For $c>1$, we refer to the representation at a node $v\in V$ $\bx \mapsto (\bar\bs_f[v],\bh_{f}[v])_{f\in\mathcal{F}}$
    as $c$-band-dominant if 
    $$
        \max \{\bar\bs_f[v]: f\in\mathcal{F}_\band\} = c \cdot \max\{\bar\bs_f[v]: f\in \mathcal{F}_\low\}.
    $$
\end{defn}

A band-dominant representation may be interpreted as a situation where the model is paying more attention to band-pass filter information at a specific node than to low-pass filter information. We now show that, under certain assumptions, the model may nevertheless primarily use the low-pass information. We provide the proof of this statement in the Appendix~\ref{sec:proof}.

\vspace{4pt}
\begin{thm}\label{thm:motivation}
 Consider the non-decoupled architecture and a $c$-band-dominant representation $\bx \mapsto (\bar\bs_f, \bh_f)_{f\in\mathcal{F}}$ at a node $v\in V$ for a non-negative  input signal $\bx\geq 0$.
    Then, if all filters have sufficiently large scales, the low-pass filter responses will always dominate the band-pass filter responses in the sense that
\begin{equation*}
        \sum_{f\in\mathcal{F}_\band} \bar\bs_f[v] \bh_f[v] \leq \epsilon c \sum_{f\in\mathcal{F}_\low} \bar\bs_f[v]\bh_f[v],
    \end{equation*}
    where $\epsilon>0$ can be made arbitrarily small if the scales of the filters are sufficiently large.
\end{thm}

For adequate choices of filters, we can therefore expect the magnitude of filter responses of low-pass filters to be significantly larger than the band-pass filter responses. This leads us to the conclusion that the non-decoupled attention limits the ability of the network to leverage diverse information. 
Our \emph{decoupled} architecture addresses this shortcoming by (i) learning the functions $\mlp(\cdot)$ for each filter in Equations~\ref{eq:low-k} and~\ref{eq:band} to individually \emph{re-balance} the filter responses, and (ii) \emph{decoupling} low-pass from band-pass filters to prevent the former from dominating the latter.

\section{Experiments}
\label{sec:expts}
\noindent\textbf{Data \& Benchmarks:}
Our choice of benchmark datasets is based on those used in previous works exploring GNNs for CO problems \citep{karalias2020erdos,zhang2023let,sun2022annealed}. For the maximum clique problem, we use the RB model from \citet{rb_graphs} to generate challenging synthetic graphs as proposed by \citet{karalias2020erdos}. 
For MDS and maximum cut, we follow \citet{zhang2023let} and \citet{sun2022annealed}, generating Barabási–Albert (BA) graphs \citep{ba_graphs} on \emph{small} (200-300 vertices) and \emph{large} (800-1200 vertices) scales. 

For benchmarks, we use Gurobi \citep{gurobi}, a state-of-the-art program solver that outperforms all deep learning-based methods in \citet{karalias2020erdos} and \citet{zhang2023let}, alongside deep-learning-based solvers and two heuristic approaches: Greedy (see Appendix~\ref{sec:greedy_heuristics} for the specific algorithms) and mean-field annealing (MFA)~\citep{mfa}, an extension of simulated annealing (SA)~\citep{sa} that replaces the Markov process formulation of SA with a mean-field to guide the search procedure.

We note that while Gurobi is very accurate when given large amounts of time, on tighter time budgets it is unable to find optimal solutions \citep{karalias2020erdos}. As for deep-learning-based solvers, we compare against \citet{karalias2020erdos} (Erd\H{o}s), \citet{zhang2023let} (GFN),  \citet{min2022can} (ScatteringClique), and \citet{sun2022annealed} (Anneal). The Erd\H{o}s models and ScatteringClique represent our main benchmarks since our model \& training framework, a GNN-based encoder-decoder with customized message-passing layers, is most analogous to them. Anneal represents a different avenue in that it proposes a novel training framework with an annealed loss function while borrowing  Erd\H{o}s's architecture. For a fairer comparison, we also compare against larger variants of Erd\H{o}s and Anneal with the number of layers and hidden dimensions set to match GCON at each task; denoted Erd\H{o}s-large and Anneal-large respectively. GFN is a generative flow network~\citep{gflownets} based solver that represents the state-of-the-art deep learning method on various problems, despite being less scalable than GNN-based solvers.

As input node features to \method, 
 we compute a collection of node-level statistics; we find that this approach is superior to using one-hot encodings (akin to  Erd\H{o}s/Anneal) or random features. For BA datasets we use node degree, eccentricity, cluster coefficient, and triangle counts; we drop eccentricity on RB graphs due to its computational cost. These features are then mapped to the hidden dimension of our GNN layers with a linear layer or a shallow MLP.

\begin{table*}[!htbp]
    \centering
    \caption{Performance comparison of \method with other baselines on \emph{small} datasets. Average of three runs with standard deviation listed. The best deep learning-based method is highlighted in \textbf{bold}, and second best is \underline{underlined}. \textsuperscript{\textdagger} indicates  Gurobi outperforms all deep learning methods on a given task.}
    \scriptsize
    \begin{tabular}{lccccccccc}
\toprule
\textbf{Method}      & \textbf{Type} & \textbf{MCut size $\uparrow$} & \textbf{Time $\downarrow$} & \textbf{MClique size $\uparrow$} & \textbf{Time $\downarrow$} & \textbf{MDS size $\downarrow$} & \textbf{Time $\downarrow$}  \\
                     &               & \textbf{BA-small}             &                            & \textbf{RB-small}           &                            & \textbf{BA-small}              &                             \\ 
\midrule
GUROBI               & OR            & \phantom{\textsuperscript{\textdagger}}732.47\textsuperscript{\textdagger}                        & 13:04                      & \phantom{\textsuperscript{\textdagger}}19.05\textsuperscript{\textdagger}                       & 1:55                       & \phantom{\textsuperscript{\textdagger}}27.89\textsuperscript{\textdagger}                          & 1:47                        \\ 
\midrule
GREEDY               & H             & 684.53 ± 1.17                      & 0:13                       & 13.54 ± 0.10                     & 0:25                       & 37.39 ± 0.19                         & 2:13                        \\
MFA               & H             & 719.78 ± 1.08                       & 1:36                       & 14.81 ± 3.64                       & 0:27                       & 36.23 ± 0.13                        & 2:56                        \\
\midrule
GFN                  & SSL           & 700.23 ± 0.60                       & 2:57                       & \textbf{16.22 ± 0.06}              & 0:42                       & \textbf{29.14 ± 0.68}                 & 2:20                        \\
ANNEAL               & SSL           & 704.28 ± 1.32                      & 0:05                       & 13.06 ± 0.15                     & 3:42                       & 66.52 ± 35.1                          & 1:50                        \\
ANNEAL-large               & SSL           & 719.57 ± 1.56                      & 0:08                       & 13.27 ± 0.12                      & 3:10                       & 32.45 ± 0.49                         & 1:20                        \\
 ERD\H{O}S               & SSL-GNN       & 704.72 ± 1.39                        & 0:05                       & 13.09 ± 0.19                      & 3:42                       & 69.92 ± 40.8                         & 1:50                        \\
 ERD\H{O}S-large               & SSL-GNN       & 720.04 ± 1.23                        & 0:08                       & 13.21 ± 0.13                     & 3:10                       & 32.48 ± 0.29                          & 1:20                        \\
ScatteringClique    & SSL-GNN       & \underline{724.45 ± 1.46}        &  0:18                      & 15.80 ± 0.12                & 4:06                       &  31.07 ± 0.12                    &  1:11                       \\
\method (Ours) & SSL-GNN       & \textbf{727.09 ± 1.49}        & 0:22                       & \underline{15.87 ± 0.15}                & 4:09                       & \underline{30.26 ± 0.30}                   & 1:15                        \\
\bottomrule
    \end{tabular}
    \label{tab:method-comparison-s}
\end{table*}

\begin{table*}[!htbp]
    \centering
    \caption{Performance comparison of \method with other baselines on \emph{large} datasets. Average of three runs with standard deviation listed. The best deep learning-based method is highlighted in \textbf{bold}, and  second best is \underline{underlined}. \textsuperscript{\textdagger} indicates  Gurobi outperforms all deep learning methods on a given task.}
    \scriptsize
    \begin{tabular}{lccccccc} 
\toprule
\textbf{Method}      & \textbf{Type} & \textbf{MCut size $\uparrow$} & \textbf{Time $\downarrow$} & \textbf{MClique size $\uparrow$} & \textbf{Time $\downarrow$} & \textbf{MDS size $\downarrow$} & \textbf{Time $\downarrow$}  \\
                     &               & \textbf{BA-large}             &                            & \textbf{RB-large}           &                            & \textbf{BA-large}              &                             \\ 
\midrule
GUROBI               & OR            & 2915.29                       & 1:05:29                    & \phantom{\textsuperscript{\textdagger}}33.89\textsuperscript{\textdagger}                       & 16:40                      & \phantom{\textsuperscript{\textdagger}}103.80\textsuperscript{\textdagger}                         & 13:48                       \\ 
\midrule
GREEDY               & H             & 2781.30 ± 3.59                     & 3:07                       & 27.05 ± 0.26                       & 0:25                       & 141.95 ± 0.34                        & 35:01                       \\
MFA                  & H             & 2929.65 ± 4.14                      & 7:16                       & 28.56 ± 0.19                       & 2:19                       & 167.05 ± 1.59                       & 36:31                       \\ 
\midrule
GFN                  & SSL           & 2826.64 ± 19.1                     & 21:20                      & \textbf{31.73 ± 1.50}              & 4:50                       & \underline{113.77 ± 1.94}                & 32:12                       \\
ANNEAL               & SSL           & 2858.85 ± 1.31                     & 0:05                       & 24.86 ± 0.63                      & 4:12                       & 317.30 ± 137                         & 31:42                        \\
ANNEAL-large               & SSL           & 2878.73 ± 21.4                     & 0:08                       & 24.06 ± 0.48                       & 4:20                       & 131.89 ± 2.48                         & 14:29                        \\
ERD\H{O}S               & SSL-GNN       & 2858.72 ± 2.78                     & 0:05                       & 25.20 ± 0.14                       & 4:13                       & 274.09 ± 95.0                         & 30:28                        \\
ERD\H{O}S-large               & SSL-GNN       & 2881.48 ± 11.8                         & 0.08                       & 23.92 ± 0.66                     & 4:19                       & 134.39 ± 2.83                        & 14:29                        \\
ScatteringClique    & SSL-GNN       & \underline{2952.03 ± 4.85}        & 0:19                       & 29.36 ± 0.38                & 4:21                       &  121.26 ± 1.68                    & 14:52                       \\
\method (Ours) & SSL-GNN       & \textbf{2961.19 ± 3.58}       & 0:27                       & \underline{29.46 ± 0.51}                & 4:57                       & \textbf{113.47 ± 0.63}                  & 11:24                        \\
\bottomrule
\end{tabular}
    \label{tab:method-comparison-l}
\end{table*}

\begin{table*}[!htbp]
    \centering
    \scriptsize
    \caption{Ablation study comparing our decoupled \method layer to the non-decoupled version and common message-passing layers from the literature for the MCut problem on the BA-small dataset. Average of three runs with standard deviation listed.}
    \begin{tabular}{lccc}
        \toprule
        \textbf{Convolution} & \textbf{MCut size $\uparrow$} & \textbf{MClique size $\uparrow$} & \textbf{MDS size $\downarrow$}  \\
            & \textbf{BA-small}  & \textbf{RB-small}  & \textbf{BA-small} \\
        \midrule
        \method (Ours, decoupled)   & \textbf{727.09 ± 1.49}  & \textbf{15.87 ± 0.15}  & \textbf{30.26 ± 0.30} \\
        \method (Ours, non-decoupled)   & 725.11 ± 1.77 &    15.78 ± 0.23 & 39.17 ± 5.47 \\
        ScatteringClique     & 724.45 ± 1.46 &  15.80 ± 0.12 & 31.07 ± 0.20 \\
        GCN     & 684.24 ± 0.20 & 15.24 ± 0.16 & 48.43 ± 1.66 \\
        GIN  & 691.86 ± 0.88 & 14.48 ± 0.33 & 33.96 ± 0.71 \\
        \bottomrule
    \end{tabular}
    \label{tab:method-comparison-mp}
\end{table*}

\noindent\textbf{Results \& Analysis:}
In Tables~\ref{tab:method-comparison-s} and~\ref{tab:method-comparison-l}, we report the mean result of the objective sizes, in addition to inference times for \emph{small} and \emph{large} datasets respectively. OR refers to algorithmic solvers, H refers to heuristics while SSL refers to the self-supervised learning algorithms we compete against. SSL-GNN denotes self-supervised methods with specialized GNN layers. OR results are as reported in ~\citet{zhang2023let}, while we evaluate all H \& SSL baselines on our synthetic datasets using identical data splits across multiple seeds. More information on our experimental setup is available in Appendix~\ref{sec:hyperparams}.

\method attains the most striking results for the maximum cut problem (MCut). In both BA-small and BA-large, we obtain the best results amongst all non-OR methods: On BA-small, we obtain a max-cut size of 727.09, a mere five short of the Gurobi solver. On BA-large, the results are arguably even more impressive: We not only attain the best result amongst non-OR methods, but also surpass the Gurobi solver by almost 45 (due to the fact that Gurobi is unable to reach an exact solution within the time constraint). The success of \method is further underlined by the fact that the inference time is fairly robust to the growing graph sizes in MCut, taking just 27 seconds to evaluate on 500 graph, whereas it takes the Gurobi solver more than an hour to do so. 

On the maximum clique (MClique) and minimum dominating set (MDS) benchmarks, \method remains very competitive. We surpass GFN by a small margin to attain the best result for MDS on large graphs, while outperforming the GNN-based solvers ScatteringClique,  Erd\H{o}s and Anneal (including the large variants) as well as the heuristic methods on all MClique and MDS benchmarks. The iterative GFlowNet algorithm proves particularly useful for MClique, and represents the best SSL method for both MClique experiments as well as MDS-BA-small. We should point out that further gains for our algorithm may be possible here by integrating the optimization framework proposed in Anneal~\citep{sun2022annealed} into our framework the meantime, by outperforming our main baselines  Erd\H{o}s and ScatteringClique on all benchmarks tested on, we demonstrate the power of \method as a model framework for a wide array of graph-CO problems.

Finally, we perform an ablation study on the message-passing layers in order to demonstrate the architectural gains of the \method layer, presented in Table~\ref{tab:method-comparison-mp}. We compare our \method with a variant \emph{without} the decoupled filter bank as well as ScatteringClique and two well-established baseline MPNNs, GCN~\citep{kipf2016semi} and GIN~\citep{xu2018powerful}, the latter of which forms the backbone of the  Erd\H{o}s' GNN architecture. We use identical depth, width, and training procedure for all models. GCON provides a clear improvement of about three cut-edges on average compared to ScatteringClique, and more than 35 compared to the best MPNN baseline (GIN) thanks to our hybrid layer. Decoupling the filter bank provides an additional two-edge improvement on average. We note that these gains are further amplified on the large counterparts.
Similar to the  Erd\H{o}s' GNN, we use skip-connections and batch normalization between layers and perform minimal hyperparameter tuning on several components like the number of layers, layer normalization, and the use of skip-connections within the hybrid layer. See  Appendix~\ref{sec:hyperparams} for further details on hyperparameter tuning.

\section{Conclusion}\label{sec:conclusion}
We have introduced \method, a novel GNN that uses a sophisticated pair of decoupled filter banks and a localized attention mechanism for solving several well-known CO problems.
\method can be trained using self-supervised loss functions, which allow us to estimate the probability $\bp(v)$ that each node is in the set of interest and then determine our solution by utilizing rule-based decoders. We then demonstrate the effectiveness of our method compared to a variety of common GNNs and other SSL methods, obtaining leading performance for the max-cut problem. We note that in the future, it would be interesting to (a) combine our architectural improvements with better optimization frameworks to leverage the power of our model better, and (b) extend \method to other CO problems such as MIS and graph coloring to further validate its efficacy. 

\paragraph{Limitations and future work.}

We found our framework to be more difficult to optimize well on larger datasets. This insight corroborates our understanding that the self-supervised loss functions associated with graph-CO problems are highly non-convex and typically challenging to optimize. This motivates future work that incorporates the annealed optimization framework from \citet{sun2022annealed} into \method for superior performance. Additionally, the primary challenge regarding extending \method to other graph-CO problems would be to develop appropriate self-supervised losses or efficient methods of producing labeled data to allow for supervised learning approaches.

\section*{Acknowledgements and Disclosure of Funding}

This work was was partially funded by the Fin-ML CREATE graduate studies scholarship for PhD, the J.A. DeSève scholarship for PhD and Guy Wolf’s research funds [Frederik Wenkel]; Bourse en intelligence artificielle des Études supérieures et postdoctorales (ESP) 2023-2024
[Semih Cantürk]; Natural Sciences and Engineering Research Council of Canada (NSERC) CGS D 569345 - 2022 scholarship [Stefan Horoi]; NSF OIA 2242769 [Michael Perlmutter]; Canada CIFAR AI Chair, IVADO (Institut de valorisation des données) grant PRF-2019-3583139727, FRQNT (Fonds de recherche du Québec - Nature et technologies) grant 299376 and NSERC Discovery grant 03267 [Guy Wolf]; NSF DMS grant 2327211 [Michael Perlmutter and Guy Wolf]. This research was also enabled in part by compute resources provided by Mila (mila.quebec). The content provided here is solely the responsibility of the authors and does not necessarily represent the official views of the funding agencies.

\bibliographystyle{plainnat}
\bibliography{references}


\appendix
\newpage

Our appendix is organized as follows. In Appendix \ref{app: theory}, we provide further details on the baseline architecture \citet{min2022can} and the proof of Theorem \ref{thm:motivation}. In Appendix~\ref{sec:hyperparams}, we give further details on our experimental setup and hyperparameters. 
In Appendices~\ref{sec:timing} and \ref{sec:generalization}, we conduct and discuss further experiments related to timing and model generalization respectively. Appendix~\ref{sec:mip} presents the MIP formulations for the graph-CO tasks that were employed in Gurobi. Finally, in Appendix~\ref{sec:greedy_heuristics} we provide the greedy heuristic algorithms used as baselines.

\section{Further Theoretical Analysis}\label{app: theory}

\subsection{Further details on the baseline architecture \citep{min2022can}}\label{sec: original}

The GNN-based method most closely related to our work is \citet{min2022can}, which also utilized a complex filter bank consisting of low-pass and band-pass filters. As alluded to in Section~\ref{sec: Attention over filters}, our method builds on \citet{min2022can} in certain ways.
(i) We utilize a learnable function $\mlp$ in our filtering operations; (ii) we use separate attention mechanisms for the aggregation operations $\mathcal{F}_A$ and the comparison operations $\mathcal{F}_C$; and (iii) our layer-wise update rule takes a different form. To clarify these differences, below, we give a more detailed recollection of the architecture from  \citet{min2022can}. 

 In \citet{min2022can}, similar to our method, the input features $\bX^{\ell-1}$ of the $\ell^{th}$ layer were filtered by each filter from a filter bank $\mathcal{F}$ that decomposes into low-pass and band-pass filters $\mathcal{F}=\mathcal{F}_\low\sqcup\mathcal{F}_\band$. The filter responses were also derived according to Equation~\ref{eq:low-k} and~\ref{eq:band}. However, in \citet{min2022can}, they use a simple, non-learned function $\mlp(.)$ that is just an activation function $\mlp=\sigma$. By contrast, we use learned functions. 

Next, \citet{min2022can} used an attention mechanism to determine the importance of information of different filters for each individual node. For each filter response $\bH_f\coloneqq F(\bX)\in \R^{n\times d}$ of $F\in \mathcal{F}$, and input features $\bX$, this takes the form
\begin{equation*}
    \bs_f \coloneqq \sigma\left( [\bX \Vert \bH_f] \ba \right) \in \R^{n\times 1},
\end{equation*}
where $\Vert$ denotes horizontal concatenations. However, differing from our method, when computing the normalized attention scores the softmax is applied over the entire filter bank. This results in a single set of normalized attention scores given by $\bar\bs_f\in \R^n$ where
\begin{equation*}
    \bar\bs_f(v) = \text{softmax}\left(\{\bs_f(v): f\in \mathcal{F}\}\right)
\end{equation*}
whereas in \method, we compute separate normalized attention scores $\bar{\bs}_f^A$ and $\bar{\bs}_f^C$ for the aggregation and comparison operators.
Next, the filter responses are re-weighted according to the attention scores, yielding $\bH_{\mathcal{F}}\in\R^{n\times d}$, where
\begin{equation*}
    \textstyle
    \bH_{\mathcal{F}} = \sum_{f\in \mathcal{F}} \left(\bar\bs_f \circ \1_d\right) \odot \bH_f,
\end{equation*}
where, as in Section~\ref{sec: Attention over filters},
 $\circ$ and $\odot$ denote the outer and Hadamard product, respectively.
The final output of each block is then generated by simply applying an MLP to the $\bH_{\mathcal{F}}$, i.e.,
\begin{equation}\label{eq:att-4}
    \bX^\ell \coloneqq \text{MLP}(\bH_{\mathcal{F}}),
\end{equation}
instead of the update rule considered in Equation~\ref{eq:att-4-}, which adds $\mathbf{H}_{\mathcal{F}_A}$
 and $\mathbf{H}_{\mathcal{F}_C}$ to $\bX^{\ell-1}$ before applying the MLP.

\subsection{Proof of Theorem~\ref{thm:motivation}}\label{sec:proof}

 We start by stating a copy of Theorem \ref{thm:motivation} presented in Section~\ref{sec:motivation} of the main text.

\vspace{4pt}
\begin{thm*}\label{thm:motivation copy}
    Consider the non-decoupled architecture and a $c$-band-dominant representation $\bx \mapsto (\bar\bs_f, \bh_f)_{f\in\mathcal{F}}$ at a node $v\in V$ for a non-negative  input signal $\bx\geq 0$.
    Then, if all filters have sufficiently large scales, the low-pass filter responses will always dominate the band-pass filter responses in the sense that
\begin{equation}\label{eq:ineq}
        \sum_{f\in\mathcal{F}_\band} \bar\bs_f[v] \bh_f[v] \leq \epsilon c \sum_{f\in\mathcal{F}_\low} \bar\bs_f[v]\bh_f[v],
    \end{equation}
    where $\epsilon>0$ can be made arbitrarily small if the scales of the filters are sufficiently large.
\end{thm*}

\begin{proof}
    
    
Let $\epsilon>0$, let $v\in V$ be fixed and set $p_\infty \coloneqq \Vert \bx\Vert_1\bd[v]/\Vert\bd\Vert_1$, where $\bd\in\R^n$ is the degree vector. As $k\rightarrow\infty$, it is known that $ \bP^k(\bx/\|\mathbf{x}\|_1)$ converges to the stationary distribution $\bd/\Vert\bd\Vert_1$, which implies that $ F_k(\mathbf{x})=\bP^k\bx$ converges to $\Vert \bx\Vert_1\bd/\Vert\bd\Vert_1$. This in turn implies that $\lim_{k_1,k_2\rightarrow\infty}F_{k_1,k_2}(\mathbf{x})=0$.

Therefore, 
     for all $\delta>0$, there exists some $K\coloneqq K(\delta)\in\N$ such that if all of the filters in $f\in\mathcal{F}_\low$ have the form $f=F_k$, $k\geq K$, and if all of the filters in $f\in\mathcal{F}_\band$ have the form $f=F_{k_1,k_2},$ $k_1,k_2\geq K$, then we have that 
     $$
     \vert\mathbf{h}_f[v]-p_\infty\vert\leq \delta \text{ for all } f\in\mathcal{F}_\low
     $$
and     
     $$
     \vert\mathbf{h}_f[v]\vert\leq \delta \text{ for all } f\in\mathcal{F}_\band.
     $$
     
    

       We next 
    set $\tau \coloneqq \max \{\bar\bs_f[v]: f\in\mathcal{F}_\band\}$, so that the left-hand-side of Equation~\ref{eq:ineq} can be estimated from above as
    $$
        \textstyle\text{LHS} = \sum_{f\in\mathcal{F}_\band} \bar\bs_f[v] \bh_f[v]\leq \sum_{f\in\mathcal{F}_\band}\tau\delta = \vert \mathcal{F}_\band\vert\tau\delta.
    $$
    At the same time, if we let $f_\low^{\max}$ denote the filter from $\mathcal{F}_\low$ corresponding to the largest value of $\bar{\mathbf{s}}_f$, we can recall the definition of a $c$-band dominate representation to estimate the right-hand-side from below by
    $$
   \text{RHS} \geq \epsilon \cdot c \cdot \bar\bs_{f_\low^{\max}} \mathbf{h}_{f_{\low}^{\max}}[v] \geq \epsilon \cdot c \cdot \frac{\tau}{c} \left(p_\infty - \delta\right)= \epsilon \cdot \tau \left(p_\infty - \delta\right). 
    $$
    
    Hence, Equation~\ref{eq:ineq} will be true if
    \begin{equation}\label{eqn: sufficient}
          \vert \mathcal{F}_\band\vert \leq \frac{p_\infty - \delta}{\delta}\epsilon.
    \end{equation}
    We now recall that for filters with sufficiently large scales (i.e., for sufficiently large $K$), $\delta$ can be arbitrarily small such that the inequality will hold. To confirm this, we note that $g(\delta) \coloneqq (p_\infty - \delta) / \delta$ is strictly increasing when decreasing $0<\delta<p_\infty$ and tends to infinity as $\delta\downarrow 0$. Thus, for sufficiently large $K$ we can choose $\delta$ small enough such that Equation~\ref{eqn: sufficient} holds.
\end{proof}

\newpage
\section{Experimental setup \& hyperparameters}\label{sec:hyperparams}

Our experimental framework is built on PyTorch Geometric~\citep{pyg} and GraphGym~\citep{you2020design}. All experiments are conducted using a single GPU and 4 CPUs; most experiments (e.g. MClique experiments on -small graphs) ran within several minutes to an hour, while the largest (e.g. MDS on BA-large using 16-layer-256-width \method) took several hours to converge.

We perform all timing experiments in Tables~\ref{tab:method-comparison-s} and~\ref{tab:method-comparison-l} following the setup in~\citet{zhang2023let}, using an NVIDIA V100 GPU for fair comparison of the inference time estimates on 500 samples. We report worst-case results, i.e., use a batch size of 1; assuming that \citet{zhang2023let} have done so despite not listing an explicit batch size for the timing experiments. Additional timing experiments comparing the \method layer with other baselines can be found in Appendix~\ref{sec:timing}.

As mentioned, the self-supervised loss functions associated with combinatorial graph problems are difficult to optimize well, due to their highly non-convex nature. We thus conducted hyperparameter optimization on several main components of our model. For each task, a separate set of hyperparameter tuning experiments were conducted for the number of layers, layer width, and activation, as well as layer normalization for layers and learning rate where appropriate. 

In addition to the configurations listed in Table~\ref{tab:hyperparams}, all tasks used Adam optimizer~\citep{adam} and a cosine annealing scheduler~\citep{cosine_annealing} with a 5-epoch warm-up period. Other hyperparameters shared by all tasks include the use of batch normalization and dropout with 0.3 probability.

Below in Table~\ref{tab:hyperparams}, GSN refers to graph size normalization where features are divided by the graph size after each layer, as per \citet{karalias2020erdos}. L2 refers to L2 normalization. Additionally, two different skip connection strategies are considered, denoted \emph{skipsum} and \emph{stack-concat}. Skipsum refers to summing the output of the previous layer to the output features; in stack-concat no actual skip connections are used, but all intermediate \method layer outputs are concatenated after the final GNN layer and passed to the post-GNN fully-connected layers accordingly. We found that stack-concat performs better overall, but skipsum is particularly helpful when the GNN inner dimension is large.

\begin{table}[!htbp]
\centering
\caption{Overview of model hyperparameters associated with the reported results for each dataset.}
\resizebox{\textwidth}{!}{
\begin{tabular}{lcccccc} 
\toprule
\textbf{}                 & \textbf{MCut}     & \textbf{MCut}     & \textbf{MClique}       & \textbf{MClique}       & \textbf{MDS}      & \textbf{MDS}       \\
\textbf{}                 & \textbf{BA-small} & \textbf{BA-large} & \textbf{RB-small} & \textbf{RB-large} & \textbf{BA-small} & \textbf{BA-large}  \\ 
\midrule
\textbf{\# Pre-GNN layers}        & 1                & 4                & 1                 & 1                 & 1                & 1                 \\
\textbf{\# GNN (Hybrid) layers}        & 16                & 16                & 20                 & 20                 & 16                & 16                 \\
\textbf{\# Post-GNN layers}        & 1                & 1                & 2                 & 2                 & 1                & 1                 \\
\textbf{Hybrid layer width}      & 32               & 32               & 32               & 32               & 256               & 256                \\
\textbf{Hybrid layer normalization} & None               & L2               & GSN                & GSN               & L2               & L2                \\
\textbf{Hybrid layer activation}       & ELU               & ELU               & GELU              & GELU              & GELU              & GELU               \\
\textbf{MLP activation}       & LReLU (0.3)               & LReLU (0.3)               & LReLU (0.01)              & LReLU (0.01)              & LReLU (0.3)              & GELU               \\
\textbf{Skip connection}       & stack-concat               & skipsum               & stack-concat              & stack-concat              & stack-concat              & skipsum               \\
\textbf{Skip conn. in Hybrid layer}       & Yes               & Yes               & Yes              & Yes              & No              & No               \\
\textbf{Learning rate}               & 1e-3              & 3e-3              & 1e-3              & 1e-3              & 3e-3              & 3e-3               \\
\textbf{\# Epochs}        & 200               & 400               & 100                & 100                & 200               & 200                \\
\textbf{\# Batch size}        & 256               & 256               & 8                & 8                & 256               & 256                \\
\textbf{Decoder $K$}        & N/A               & N/A               & 10                & 10                & 1               & 1                \\
\bottomrule
\end{tabular}}
\label{tab:hyperparams}
\end{table}

\paragraph{Further discussion on benchmarks}
Our main benchmarking pipeline (Tables~\ref{tab:method-comparison-s} and~\ref{tab:method-comparison-l}) parallel \citet{zhang2023let} to a considerable extent in that we evaluate GCON on the same tasks (MCut, MClique and MDS), using datasets drawn from identical distributions, i.e. RB \citep{rb_graphs} and BA \citep{ba_graphs} graphs generated with identical parameters; we also use their GFN model and several of the baselines from their study as methods to compare against in our study. Gurobi results in Tables~\ref{tab:method-comparison-s} and~\ref{tab:method-comparison-l} are as reported in ~\citet{zhang2023let}, while we evaluate all H \& SSL baselines on our synthetic datasets using identical data splits across multiple seeds. We adapted the Greedy and MFA methods from Anneal~\citep{sun2022annealed} for MClique and MDS, while implementing them ourselves for MCut, since it was a problem not considered in \citet{sun2022annealed}.

We additionally provide a brief discussion on differences between our results and those in \citet{zhang2023let} to allow for better interpretability of our results. 

\begin{itemize}
    \item We observed that model depth is particularly important for GNN-based methods on MCut and MDS, as we see our deeper GNN variants Erd\H{o}s-large and Anneal-large outperform their shallower counterparts significantly. We conjecture that \citet{zhang2023let} did not evaluate on deeper variants of ERD\H{O}S/ANNEAL on MCut, which may explain why these methods underperform in their paper. On the other hand, MClique did not benefit substantially from model depth, where the large and base models performed very similarly. Shallow models particularly struggled on MDS, and exhibited extreme variance across runs.
    \item Annealing was less helpful in our case than suggested in \citet{sun2022annealed} and \citet{zhang2023let} on MClique. This is likely due to the fact that we were able to optimize Erd\H{o}s' GNN better, as both Erd\H{o}s and Anneal results converged to values similar to those corresponding to Anneal in \citet{zhang2023let}. We also noted that Erd\H{o}s' and Anneal were difficult to tune for MDS in general, and fell well short of the results presented in \citet{zhang2023let}. This applied to MClique as well, albeit to a lesser extent. We believe their implementation of Erd\H{o}s' GNN is based on \citet{sun2022annealed}, which diverge significantly from that of \citet{karalias2020erdos}, e.g. does not use masking. For consistency across all tasks, we modeled our implementation on \citet{karalias2020erdos}; but note that the \citet{sun2022annealed} implementation performed marginally better for MClique (though still significantly worse than \method).
    \item Heuristic results were very consistent with those in \citet{zhang2023let} for MClique and MDS, but our MFA implementation for MCut was seemingly more powerful than theirs -- while the Greedy results were similar (with a 0.5\% drop on BA-small and 0.7\% rise on BA-large), our MFA results improved over theirs by 2.2\% and 3.4\% respectively. While we do not think these differences make any substantial difference, we still report them for completeness, particularly as we do not have access to the exact implementations nor parameters used in their work.
\end{itemize}


\newpage
\section{Timing studies}\label{sec:timing}

We also performed a set of timing experiments to provide an overview of the scalability of the \method layer compared to ScatteringClique and GCN, two GNN layers benchmarked against in this study. Using the configurations denoted in Table~\ref{tab:hyperparams}, we ran our framework on 500 test graphs for each dataset. We see that even when the considered datasets are identical, as in the case of MCut and MDS (both tasks use the BA datasets), both the duration and scalability of the test runs are largely dependent on the task-specific decoder: The MCut decoder is faster than the MDS decoder across all three models; perhaps more remarkably, while the MCut decoder scales very well from BA-small to BA-large, the runtimes for the MDS decoder increase by 1000+\% for all models between the two datasets. This implies that for GNN-based neural solvers, larger efficiency gains are likely to be made by focusing on optimizing the decoder architectures than the GNN layers themselves.

Decoder scalability aside, our \method layer maintains similar runtimes with ScatteringClique and GCN across all experiments, indicating the overall scalability of the \method model. We note that there is \emph{some} overhead associated with our filters and attention mechanism such that GCN remains the fastest model in most cases, one exception is MDS on BA-large where our layer is faster than both baselines.

\begin{table}[!htbp]
\centering
\caption{Timing comparison of \method, ScatteringClique and GCN on Small and Large datasets for the MCut and MClique tasks. The times are based on test datasets of 500 graphs.}
\begin{tabular}{lccr} 
\toprule
\textbf{Model}                     & \textbf{Small}     & \textbf{Large}       & \textbf{\% Increase (s)}       \\
\midrule
       & \textbf{MCut -- BA} & \textbf{MCut -- BA} & \\ 
\midrule
\method                & 0:22           & 0:27           & 22.73             \\
ScatteringClique               & 0:18           & 0:19           & 5.56             \\ 
GCN               & 0:14           & 0:14           & 0.00             \\ 
\midrule
       & \textbf{MDS -- BA} & \textbf{MDS -- BA} & \\ 
\midrule
\method                & 1:15           & 11:24           & 812.00             \\
ScatteringClique               & 1:11           & 14:52           & 1156.34             \\ 
GCN               &  1:14          & 13:36           & 1002.70             \\ 
\midrule
       & \textbf{MClique -- RB} & \textbf{MClique -- RB} & \\ 
\midrule
\method                & 4:09           & 4:57           & 19.28             \\
ScatteringClique               & 4:06           & 4:21           & 6.10             \\ 
GCN               & 3:13           & 4:04           & 26.42             \\ 
\bottomrule
\end{tabular}
\label{tab:timing_exps}
\end{table}

\newpage

\section{Generalization study}\label{sec:generalization}

We also present a brief study on the transferability of our \method architecture. As every graph-CO problem in our study has a corresponding \emph{small} and \emph{large} benchmark dataset, we trained a \method model on each, and tested the respective models on the \emph{other} dataset (e.g. for MCut, we trained a \method on BA-small and tested it on BA-large, and vice versa), and measured the percentage change in performance compared to the in-distribution case (e.g. \method both trained and tested on BA-small). We repeated the study for GFN and Erd\H{o}s-large as well for comparison. Our results are presented in Table~\ref{tab:generalization}, while Table~\ref{tab:agg_generalization} shows average changes across several groupings such as type of generalization and generalization per graph-CO problem.

Both \method and GFN prove to be quite transferable across all our benchmarks, with less than 6\% average drop in performance for \method and only a 2\% drop for GFN. Nevertheless, the generalizability patterns of the two methods are considerably different. \method is consistently better when generalizing from \emph{small} to \emph{large} datasets while it struggles more in the opposite direction; these trends are completely reversed for GFN, which does not exhibit any negative transfer when generalizing from \emph{large} to \emph{small}. When looking at individual CO problems, \method is particularly robust on MCut in both directions while GFN is more robust for MClique and MDS. Erd\H{o}s-large is on par with both methods on MCut and MClique, but is vastly inferior with huge performance drops in MDS. We note that despite the marginal drops, \method still comfortably outperforms most in-distribution baselines even in the transfer learning setting.

\begin{table}[!ht]
    \centering
    \small
    \caption{Generalization results between \emph{small} and \emph{large} graphs for \method and GFN across all tasks. Average of three runs listed.}
    \begin{tabular}{ccccccc}
    \toprule
        \textbf{Task} & \textbf{Model} & \textbf{Train Data} & \textbf{Test Data} & \textbf{Base Perf.} & \textbf{Transfer Perf.} & \textbf{\% Change} \\ \midrule
        MCut & GFN & BA-large & BA-small & 700.23 & 704.99 & \cellcolor[HTML]{e4f4e8}0.68\% \\
        MCut & GCON & BA-large & BA-small & 727.09 & 723.45 & \cellcolor[HTML]{fefbfb}-0.50\%\\
        MCut & ERD\H{O}S-large & BA-large & BA-small & 720.04 & 664.82 & \cellcolor[HTML]{f1c2c2}-7.67\% \\ \midrule
        MCut & GFN & BA-small & BA-large & 2826.64 & 2569.44 & \cellcolor[HTML]{f1c2c2}-9.10\% \\
        MCut & GCON & BA-small & BA-large & 2961.19 & 2951.18 &\cellcolor[HTML]{fefbfb} -0.34\% \\
        MCut & ERD\H{O}S-large & BA-small & BA-large & 2881.48 & 2895.51 & \cellcolor[HTML]{e4f4e8}0.49\% \\ \midrule
        MClique & GFN & RB-large & RB-small & 16.22 & 16.26 & \cellcolor[HTML]{f5fbf7}0.25\% \\ 
        MClique & GCON & RB-large & RB-small & 15.87 & 13.62 & \cellcolor[HTML]{eaa1a0}-14.17\% \\
        MClique &ERD\H{O}S-large &RB-large &RB-small &13.21 &13.00 &\cellcolor[HTML]{f9edeb}-1.56\%\\ \midrule
        MClique & GFN & RB-small & RB-large & 31.73 & 30.16 & \cellcolor[HTML]{f8dede}-4.95\% \\
        MClique & GCON & RB-small & RB-large & 29.46 & 28.57 & \cellcolor[HTML]{faeaea}-3.01\% \\
        MClique & ERD\H{O}S-large & RB-small & RB-large & 23.91 & 26.27 & \cellcolor[HTML]{58a65c}\textcolor{white}{9.85\%}
        \\ \midrule
        MDS & GFN & BA-large & BA-small & 29.14 & 29.13 & 0.02\% \\
        MDS & GCON & BA-large & BA-small & 30.26 & 33.89 & \cellcolor[HTML]{edafae}-11.99\% \\ 
        MDS & ERD\H{O}S-large & BA-large & BA-small & 32.48 & 46.82 & \cellcolor[HTML]{d96255}\textcolor{white}{-44.14\%} \\ \midrule
        MDS & GFN & BA-small & BA-large & 113.77 & 112.89 & \cellcolor[HTML]{e0f2e5}0.77\% \\
        MDS & GCON & BA-small & BA-large & 113.47 & 118.70 & \cellcolor[HTML]{f8dede}-4.60\% \\ 
        MDS & ERD\H{O}S-large & BA-small & BA-large	& 134.39 & 362.68 & \cellcolor[HTML]{d96255}\textcolor{white}{-169.87\%} \\ \bottomrule
    \end{tabular}
    \label{tab:generalization}
\end{table}

\begin{table}[!htb]
    \centering
    \caption{Aggregated \% performance changes reported in Table~\ref{tab:generalization} for \method and GFN.}
    \begin{tabular}{cccc}
    \toprule
        \textbf{Group} & \textbf{GFN} & \textbf{GCON} & \textbf{ERD\H{O}S-large} \\ \midrule
        Overall &\cellcolor[HTML]{f9edeb} -2.05\% &\cellcolor[HTML]{f8dede} -5.77\% & \cellcolor[HTML]{d96255} \textcolor{white}{-35.48\%}\\ \midrule
        Small $\rightarrow$ Large &\cellcolor[HTML]{f2dbd8} -4.42\% & \cellcolor[HTML]{f9edeb}-2.65\% & \cellcolor[HTML]{d96255} \textcolor{white}{-53.18\%} \\ 
        Large $\rightarrow$ Small &\cellcolor[HTML]{f6fbf8} 0.32\% & \cellcolor[HTML]{f1c2c2} -8.89\% & \cellcolor[HTML]{eaa1a0} -17.79\%\\ \midrule
        MCut &\cellcolor[HTML]{f2dbd8} -4.21\% &\cellcolor[HTML]{fefbfb} -0.42\% & \cellcolor[HTML]{faeaea}-3.59\%\\ 
        MClique & \cellcolor[HTML]{f9edeb}-2.35\% &\cellcolor[HTML]{f1c2c2} -8.59\% & \cellcolor[HTML]{88c5a1}4.15\%\\ 
        MDS &\cellcolor[HTML]{f6fbf8} 0.40\% &\cellcolor[HTML]{f1c2c2} -8.30\% & \cellcolor[HTML]{d96255}\textcolor{white}{-107.01\%} \\ \bottomrule
    \end{tabular}
    \label{tab:agg_generalization}
\end{table}

\newpage
\section{MIP formulations}\label{sec:mip}
\subsection*{MCut}
$$
\begin{array}{ll}
\text { Maximize } & \sum_{(i, j) \in E} x_i + x_j - 2 x_i x_j \\
\text { subject to } & x_i \in \{ 0,1 \} \quad \forall i \in V
\end{array}
$$
\subsection*{MClique}
$$
\begin{array}{ll}
\text { Maximize } & \sum_{i \in V} x_i \\
\text { subject to } & x_i+x_j \leq 1 \quad \forall(i, j) \notin E, \\
& x_i \in\{0,1\} \quad \forall i \in V
\end{array}
$$
\subsection*{MDS}
$$
\begin{array}{ll}
\text { Minimize } & \sum_{i \in V} x_i \\
\text { subject to } & x_i+ \sum_{j \in \mathcal{N}(i)} x_j \geq 1 \quad \forall i \in V, \\
& x_i \in\{0,1\} \quad \forall i \in V
\end{array}
$$

\section{Greedy heuristic algorithms}\label{sec:greedy_heuristics}

\begin{algorithm}
\caption{Heuristic algorithm for MCut, inspired by \citet{kernighan_lin}}
\begin{algorithmic}[1]
    \Function{MCutHeuristic}{$V, E$} \Comment{Input: vertices $V$, edges $E$; Output: set of cut edges $S$}
        \State $(S_1, S_2) \gets \Call{RandomPartition}{V}$ \Comment{Randomly partition vertices into two sets}
        \ForEach {$v \in V$} \Comment{Iteratively optimize the partition}
            \State $c_1 \gets \Call{CountEdges}{v, S_1, E}$ \Comment{Edges from $v$ to nodes in $S_1$}
            \State $c_2 \gets \Call{CountEdges}{v, S_2, E}$ \Comment{Edges from $v$ to nodes in $S_2$}
            \If {$v \in S_1 \textbf{ and } c_1 > c_2$} \Comment{Move $v$ to $S_2$ if it improves the cut}
                \State \Call{Move}{$v, S_1, S_2$}
            \ElsIf {$v \in S_2 \textbf{ and } c_2 > c_1$} \Comment{Move $v$ to $S_1$ if it improves the cut}
                \State \Call{Move}{$v, S_2, S_1$}
            \EndIf
        \EndFor
        \State $S \gets \{\}$ \Comment{Initialize the set of cut edges}
        \ForEach {$\{v_i, v_j\} \in E$} \Comment{Identify edges crossing the cut}
            \If {($v_i \in S_1 \textbf{ and } v_j \in S_2$) \textbf{ or } ($v_i \in S_2 \textbf{ and } v_j \in S_1$)}
                \State \Call{Add}{$\{v_i, v_j\}, S$}
            \EndIf
        \EndFor
        \State \Return $S$ \Comment{Return the set of edges in the maximum cut}
    \EndFunction
\end{algorithmic}
\end{algorithm}

\begin{algorithm}
\caption{Heuristic algorithm for MClique \citep{mclique_heuristic}, implementation based on \citet{sun2022annealed}}
\begin{algorithmic}[1]
    \Function{MCliqueHeuristic}{$V, E$} \Comment{Input: vertices $V$, edges $E$; Output: clique $S$}
        \State $S \gets \{\}$ \Comment{Initialize the clique as an empty set}
        \State $vList \gets \Call{SortByDegree}{V, E, descending}$ \Comment{Sort vertices by descending degree}
        \For {$v_i \in vList$} \Comment{Iterate through vertices in sorted order}
            \If {$S = \{\}$} \Comment{If the clique is empty, start with the first vertex}
                \State $S \gets \{v_i\}$
            \EndIf
            \If {\Call{MaintainsClique}{$v_i, S, E$}} \Comment{Check if adding $v_i$ to $S$ maintains a valid clique}
                \State \Call{Add}{$v_i, S$} \Comment{Add $v_i$ to the clique}
            \EndIf
        \EndFor
        \State \Return $S$ \Comment{Return the maximum clique found}
    \EndFunction
\end{algorithmic}
\end{algorithm}

\begin{algorithm}
\caption{Heuristic algorithm for MDS, based on \citet{sun2022annealed}}
\begin{algorithmic}[1]
    \Function{MDSHeuristic}{$V, A$} \Comment{Input: vertices $V$, adjacency matrix $A$; Output: dominating set $S$}
        \State $S \gets \{\}$ \Comment{Initialize the dominating set as empty}
        \State $vList \gets \Call{SortByDegree}{V, A, descending}$ \Comment{Sort vertices by descending degree}
        
        \For {$v_i \in vList$} \Comment{Initialize probability of each node being in the dominating set}
            \State $p_i \gets 0.5$
        \EndFor
        \For {$v_i \in vList$} \Comment{Initialize probability of node being ``uncovered'' by the dominating set}
            \State $u_i \gets (1 - p_i) \cdot \prod_{j=1}^n \big(A_{ij} \cdot (1 - p_j) + (1 - A_{ij})\big)$
        \EndFor
        
        \For {$v_i \in vList$} \Comment{Iterate over sorted vertices}
            \State $potential \gets \frac{10}{1-p_i} \cdot \big(u_i + \sum_{j \in \mathcal{N}(i)} u_j\big)$ \Comment{Calculate potential of $v_i$}
            \If {$potential > 1$} \Comment{Add $v_i$ to the dominating set if potential exceeds threshold}
                \State \Call{Add}{$v_i, S$} \Comment{Add $v_i$ to dominating set $S$}
                \State $p_i \gets 1$ \Comment{Set probability of $v_i$ being in the dominating set to 1}
                \State $u_i \gets 0$ \Comment{Mark $v_i$ as covered}
                \State $u_j \gets 0 \quad \forall j \in \mathcal{N}(i)$ \Comment{Mark neighbors of $v_i$ as covered}
            \Else
                \State $p_i \gets 0$ \Comment{Set $v_i$ as non-dominant}
                \State $u_i \gets \frac{u_i}{1 - p_i}$ \Comment{Adjust uncover probability of $v_i$}
                \State $u_j \gets u_j \cdot (1 - A_{ij}) + A_{ij} \cdot \frac{u_j}{1 - p_i} \quad \forall j$ \Comment{Adjust neighbors' uncover probabilities}
            \EndIf
        \EndFor
        
        \State \Return $S$ \Comment{Return the dominating set}
    \EndFunction
\end{algorithmic}
\end{algorithm}

\end{document}